%% file: main.tex
\documentclass[10pt,twocolumn]{article}
\usepackage{float}
\usepackage{colm2024_conference}
\usepackage{amsmath,amssymb,amsthm}
\usepackage{graphicx}
\usepackage{adjustbox}
\usepackage{booktabs}
\usepackage{tabularx}
\usepackage[font=small]{caption}
\usepackage{placeins}
\usepackage{listings}
\usepackage{enumitem}
\usepackage{xurl}
\usetikzlibrary{arrows.meta,shapes.geometric}

\floatstyle{ruled}
\newfloat{algorithm}{tbp}{loa}
\floatname{algorithm}{Algorithm}

\providecommand{\Description}[2][]{}

\newcommand{\sys}{{\textsc{QwenGyre}}}
\newcommand{\xlong}{{\textsc{x}}Long}
\definecolor{red}{HTML}{AA0000}

\input{content/authors}

\makeatletter
\def\@maketitle{%
  \vbox{\hsize\textwidth
    \centering
    {\Large\bfseries\@title\par}
    \bigskip
    {\large\@author\par}
    \bigskip
  }%
  \thispagestyle{firstpage}%
}
\makeatother
\renewenvironment{abstract}{\section*{\abstractname}}{\par}
\fancypagestyle{firstpage}{%
  \fancyhf{}
  \fancyhead[L]{\includegraphics[height=20pt]{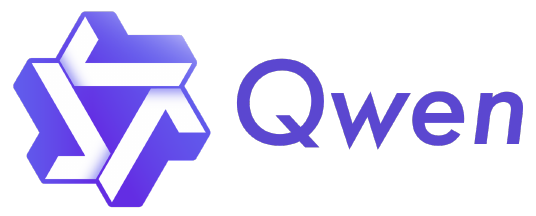}}
  \fancyhead[R]{\textcolor[HTML]{4D4D4D}{\today}}
  \fancyfoot[C]{\thepage}
}

\title{\sys{}: An Elastic Reinforcement Learning Framework for\\ Training \xlong-Horizon Agents}
\author{\PaperAuthorBlock}
\date{}
\hypersetup{
  pdftitle={\sys{}: An Elastic Reinforcement Learning Framework for Training \xlong-Horizon Agents},
  pdfauthor={\PaperAuthorNames}
}

\begin{document}
\twocolumn[%
\maketitle

\begin{abstract}
Large language model (LLM) agents increasingly undertake e\textsc{x}treme-long (\xlong{}) horizon tasks, where a single execution can span hours, hundreds of model--environment interactions, and nearly 1M tokens per rollout.
Applying online reinforcement learning (RL) to such executions poses two fundamental challenges:
(1) severe execution variance and prolonged rollout delays cause massive GPU idling;
and (2) complex non-linear branching generates massive trajectory redundancy, crippling training efficiency.

To address these, we presents \sys{}, an end-to-end framework for \xlong-horizon online RL.
\sys{} elastically reallocates GPUs between rollout and training without interrupting live executions, while its trajectory processor reconstructs branching histories, scores partial progress, and deduplicates redundant paths to bound training costs.
Scaled to our flagship model, Qwen~3.8 2.4T, with 700K tokens per rollout, \sys{} yields a 6.0\% absolute gain on NL2RepoBench (52.5\% $\to$ 58.5\%) in 48 steps. Across our evaluations on diverse domains of training datasets, \sys{}
delivers up to $1.85\times$ and $1.78\times$ speedups over Colocate
and Async, respectively.
\end{abstract}

\vspace{1cm}

  \begin{center}
    {\includegraphics[width=\textwidth]{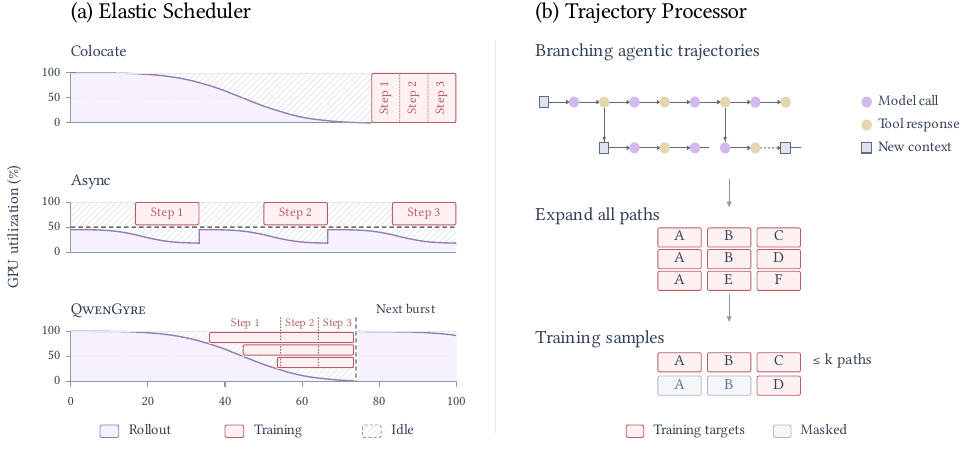}}
    \captionof{figure}{\sys{} addresses two challenges of \xlong-horizon agentic RL.
  (a) The Elastic Scheduler reallocates idle rollout capacity to training
  while continuing to serve ongoing executions.
  (b) The Trajectory Processor converts branching executions into bounded
  training samples while masking repeated targets in shared prefixes.}
    \label{fig:gpu-utilization}
  \end{center}
]

\PaperAuthorFootnotes
\null
\clearpage

\input{content/introduction}
\input{content/background_and_motivation}
\input{content/horizongyre_overview}
\input{content/elastic_scheduling_plane}
\input{content/trajectory_processing_plane}
\input{content/experimental_results}
\input{content/related_work}
\input{content/limitations}
\input{content/conclusion}

\bibliographystyle{colm2024_conference}
\bibliography{biblio}

\clearpage
\appendix
\input{content/appendix_staleness_analysis}
\clearpage
\input{content/appendix_rl_algorithm}
\clearpage
\input{content/appendix_case_study}

\end{document}

%% file: content/authors.tex
\newcommand{\PaperAuthorNames}{Weiqi Wang, Yuxin Zhou, Mouxiang Chen,
  Siyuan Zhang, Yi Zhang, Yuyan Luo, Zhiyu Yin, Chencan Wu,
  Jiemin Jiang, Wentao Yao, Chujie Zheng, JianWei Zhang}

\newcommand{\PaperAuthorName}[3]{%
  \href{mailto:#2}{\textcolor{black}{#1}}\textsuperscript{#3}}

\newcommand{\PaperEqualContributionNote}{Equal contribution.}
\newcommand{\PaperCorrespondingAuthorNote}{Corresponding author.}

\newcommand{\PaperAuthorFootnotes}{%
  \begingroup
  \renewcommand{\thefootnote}{\fnsymbol{footnote}}%
  \footnotetext[1]{\PaperEqualContributionNote}%
  \footnotetext[2]{\PaperCorrespondingAuthorNote}%
  \endgroup
}

\newcommand{\PaperAuthorBlock}{%
  \begingroup
  \normalfont\centering
  {\large
    \begin{tabular}{@{}cccc@{}}
      \PaperAuthorName{Weiqi Wang}{wangweiqi0329@mail.ustc.edu.cn}{*} &
      \PaperAuthorName{Yuxin Zhou}{suishi.zyx@alibaba-inc.com}{*} &
      \PaperAuthorName{Mouxiang Chen}{chenmx@zju.edu.cn}{*} &
      \PaperAuthorName{Siyuan Zhang}{18301188748@163.com}{} \\
      \PaperAuthorName{Yi Zhang}{1109276519@qq.com}{} &
      \PaperAuthorName{Yuyan Luo}{luoyuyan.lyy@alibaba-inc.com}{} &
      \PaperAuthorName{Zhiyu Yin}{yinzhiyu.yzy@alibaba-inc.com}{} &
      \PaperAuthorName{Chencan Wu}{wuchencan.wcc@alibaba-inc.com}{} \\
      \PaperAuthorName{Jiemin Jiang}{jiangjiemin.jjm@alibaba-inc.com}{} &
      \PaperAuthorName{Wentao Yao}{wenyao.ywt@alibaba-inc.com}{} &
      \PaperAuthorName{Chujie Zheng}{zhengchujie.zcj@alibaba-inc.com}{} &
      \PaperAuthorName{JianWei Zhang}{zhangjianwei.zjw@alibaba-inc.com}{\textdagger}
    \end{tabular}\par}
  \medskip
  {
  \small
    Alibaba Token Hub, Alibaba Group}
  \endgroup
}


%% file: content/introduction.tex
\section{Introduction}
\label{sec:intro}

Large language model (LLM) agents increasingly solve software
engineering tasks that are \emph{long-horizon}: completing one task
requires an extended sequence of interdependent decisions and
environment interactions over a persistent workspace. 
A representative class of such tasks is issue resolution (e.g., SWE-bench~\citep{swebench,swebenchpro}),
which localizes a single fix within an existing repository and spans tens of
interactions.
On these workloads, online reinforcement learning (RL) has demonstrated
promising gains by iteratively refining policies through environment interactions~\citep{deepswe_rl,golubev2025multiturn,swemaster,legorl}.

As foundation models rapidly advance in reasoning and context capabilities,
the paradigm is shifting toward fully autonomous, end-to-end software
development, such as repository generation~\citep{nl2repobench}, system reimplementation~\citep{programbench}, and codebase-wide refactoring~\citep{swerefactorbench}.
These workloads extend long-horizon execution to multi-hour, project-scale runs.
A single stateful rollout spans hundreds of dependent model--environment interactions and processes around 1M tokens across its context.
We refer to this regime as \emph{\xlong{}-horizon}.
Scaling RL to \xlong{}-horizons introduces fundamental
challenges on both system scheduling and trajectory modeling.

\textbf{Challenge 1: Extreme rollout variance and severe GPU underutilization at \xlong{}-horizons.}
Online RL conventionally adopts a \emph{colocate} design that alternates
the GPU pool between rollout and training~\citep{hybridflow,real}.
At \xlong{}-horizons, however, rollouts exhibit severe long-tail variance;
without general support for checkpointing workspace state, training remains
blocked until all live executions end, leaving GPUs idle on a few trailing tasks.
Alternative {\emph{async}} deployments statically partition GPUs into
dedicated pools~\citep{areal}.
Nevertheless, at \xlong-horizons, the wait for a ready batch can far exceed the duration of a training
step, leaving training GPUs idle while rigid partitions prevent either
stage from utilizing the other's spare capacity.
The challenge is therefore to reallocate GPUs dynamically as demand changes,
using idle capacity for training while maintaining inference service for
ongoing rollouts.

\textbf{Challenge 2: Complex trajectory branching and prohibitive training overhead at \xlong{}-horizons.}
The linear agent loop used in simple agent designs~\citep{react,sweagent}
does not survive at \xlong{}-horizons: sustaining hours of execution
typically relies on black-box harnesses that routinely compact
saturated histories, delegate sub-tasks to sub-agents with isolated contexts,
and retry failed execution paths~\citep{agentlightning_v1,clawgym2}.
As a result, a single task execution no longer yields a sequential trajectory,
but an intricate, non-linear graph of divergent context paths.
This non-linearity leads to an explosion in sample and token volume:
expanding all these paths into individual training samples introduces severe
redundancy across shared prefixes and incurs prohibitive training overhead.
The challenge is thus to reconstruct valid training representations from
such complex, branching executions while preventing the sheer volume of paths
and tokens from overwhelming the training pipeline.

In this paper, we present \sys{}, an end-to-end
framework for online RL
through {black-box} agent harnesses at
\xlong{}-horizons.
Its design separates the lifetimes of
harness executions,
GPU roles, and training samples. Harness
state persists
across GPU role changes, while recorded
model calls are
materialized into training samples on
demand. 

To address Challenge 1, \sys{} introduces an \emph{elastic scheduler}, which
adjusts GPU allocation between rollout and training according to
the amount of unfinished rollout work while ensuring that model
calls from ongoing harness executions continue to be served.
Centralized data-parallel training allows newly available
nodes to join an ongoing batch as the training pool grows.
Streaming training distributes work according to each participating
data-parallel group's progress, helping the groups finish the same
batch at similar times to minimize idle gaps.

To address Challenge 2, \sys{} introduces a \emph{trajectory processor}, which
  organizes each execution's recorded model calls into a trajectory
  tree with shared prefixes, preserving each output's original
  conditioning context.
  Task-specific evaluation scores preserved work after termination,
  including assessable partial progress after timeout.
  Once valid rewards are available, \sys{} selects a bounded set of trajectories per execution by role priority to control training cost.
  Training uses only policy-generated targets~\citep{tito}, counts shared targets once, and averages token losses within each execution to avoid overweighting executions with more paths.

We evaluate \sys{} on NL2RepoBench~\citep{nl2repobench},
DeepSWE~\citep{deepswe}, and TerminalBench~\citep{terminalbench}
using Qwen~3.6 122B. We also apply \sys{} to
\xlong{} RL training of our flagship model, Qwen~3.8 2.4T,
on NL2RepoBench.
We compare against Colocate and Async
baselines using equal GPU budgets 
{under the same staleness constraints}.
Across the reported configurations, \sys{} achieves end-to-end
speedups of 1.38 to 1.78 times over Async and 1.21 to 1.85 times
over Colocate while matching baseline training scores.
On NL2RepoBench, \sys{} improves Qwen~3.8 2.4T's score from
approximately 52.5 \% to 58.5 \% over 48 training steps.

In summary, this paper makes the
following contributions:

\begin{itemize}[leftmargin=*]
  \item We identify the scheduling and
  trajectory-processing
  challenges of online RL for \xlong{}
  agent executions.
  \item We design an elastic scheduler that reallocates GPUs between
    rollout and training as demand changes, lets newly available GPUs
    join an ongoing training batch, and balances training work across
    them to reduce idle time.
  \item We develop a trajectory processor that preserves each model
    output's original context and evaluates partial work after timeouts.
    It selects a limited number of trajectories per execution, counts
    shared outputs once, and averages token losses within each execution.
  \item We demonstrate \sys{} on Qwen~3.8 2.4T,
    improving NL2RepoBench passrate from 52.5\% to 58.5\% in
    48 training steps. Across evaluated workloads, \sys{} achieves
    up to $1.85\times$ and $1.78\times$ end-to-end speedups over
    Colocate and Async, respectively.
\end{itemize}

%% file: content/background_and_motivation.tex
\section{Background and Motivation}
\label{sec:background}

\subsection{\xlong{} Agentic RL Workloads}
\label{sec:background-workload}

{We first describe the group-relative RL pipeline and the
execution characteristics of \xlong{} agentic workloads. We then
examine their implications for rollout--training scheduling and
training data construction.}

\begin{figure}[t]
\centering
\includegraphics[width=\columnwidth]{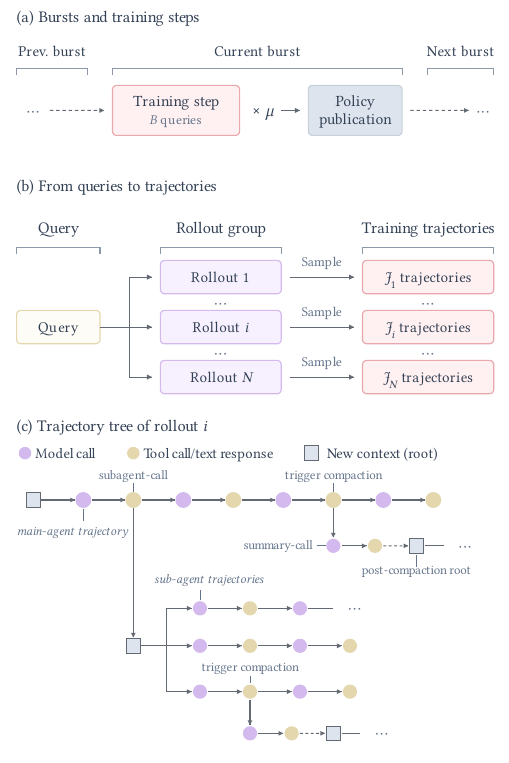}
\caption{Hierarchy of training bursts, rollout groups, and trajectories.
(a) Each burst performs $\mu$ consecutive training steps, then publishes the policy; each step consumes $B$ query groups.
(b) Each query produces $N$ independent rollouts, and rollout $i$
yields $J_i$ training trajectories.
(c) A rollout includes main-agent and sub-agent trajectories,
with compaction producing summary branches. Dashed links reinsert summaries into post-compaction contexts.}
\label{fig:data-hierarchy}
\end{figure}

\paragraph{Group-relative RL pipeline.}

For each \emph{query}, $N$ independent rollouts sampled from the policy
form a \emph{rollout group} (Figure~\ref{fig:data-hierarchy}(b)). Advantages are computed
from the group's scalar rewards, avoiding a learned value function
\citep{deepseekmath}. In our pipeline, each \emph{training step}
consumes a \emph{batch} of $B$ distinct rollout groups and advances the policy version by one. To amortize the cost of weight publication, the RL framework may
perform $\mu$ consecutive training steps between publications;
we call this interval a \emph{burst}. Each burst consumes $\mu B$
rollout groups and publishes the weights at its final step,
as shown in Figure~\ref{fig:data-hierarchy}(a).

\paragraph{Harness-driven agentic RL}
An agentic rollout is an end-to-end task execution in which a model
interacts with an environment through a harness. The harness
maintains task state, invokes tools according to model outputs,
and constructs context for subsequent model calls.
We call each model call--response interaction a \emph{model calls},
each tool invocation a \emph{tool call}, and the full task lifecycle,
including retries, a \emph{rollout execution}.
Figure~\ref{fig:data-hierarchy}(c) shows an execution that branches at tool calls
and the trajectories selected from it for training.

In our black-box setting, the harness accesses the model through
a proxy provided by the RL framework. The framework records model calls
and receives task-level evaluation results without changing the
harness's execution logic \citep{agentlightning_v1,clawgym2}.
This interface provides no general mechanism to checkpoint and
resume a live execution. If a model call cannot be served, the
harness may stall or time out, changing the task outcome for
reasons unrelated to the policy's decisions. Scheduling must
therefore preserve inference access throughout the execution.
Requests can be routed to different rollout instances while the harness
continues running.

\paragraph{The \xlong{} execution profile.}

Recent software engineering benchmarks evaluate agents on repository
construction, program reimplementation, and migration across entire
codebases \citep{nl2repobench,programbench,swemarathon,swerefactorbench}.
We focus on \xlong{} executions that last hours, span hundreds of
dependent model--environment interactions, and process around 1M
input and output tokens across model calls. This lifecycle volume
counts repeated contexts at each request and can exceed the context
limit of an individual request.

These workloads exhibit both long execution tails and clustered
terminations.
Execution durations vary across queries and within each rollout group:
generation length and tool latency affect the time per interaction
\citep{heddle}.
Over hundreds of dependent interactions, these differences can leave
some executions running long after their peers finish.
Since group-relative advantages require all $N$ rewards, a straggler
can delay training on its entire group.
Executions with similar start times and a common timeout budget may expire within a narrow interval, producing clusters of execution terminations that can sharply reduce inference demand.

\subsection{Rollout--Training Scheduling}
\label{sec:background-scheduling}

RL frameworks organize rollout and training through two scheduling
policies. A data \emph{dispatch policy} controls when and how much
rollout work enters the pipeline. A GPU \emph{placement policy}
allocates GPUs between rollout and training in space and time.
These policies can be combined to form different framework designs:
Async, for example, can use either continuous or
boundary dispatch.

\paragraph{Data dispatch policy and staleness.}

Under \emph{boundary dispatch}, each rollout group enters the dispatch
queue when admitted and retains its slot until training consumes it.
Let $\varphi\geq0$ denote the extra dispatch in units of $B$ rollout
groups. With $\mu$ training steps per burst, the target queue depth is
$(\varphi+\mu)B$ groups. The queue is filled to this depth at startup
and replenished with $\mu B$ new groups after each burst's weight
publication.

We measure a group's \emph{scheduling staleness} as $d=v_t-v_d$,
with policy versions indexed by training steps. Here, $v_d$ is the
version of the latest published policy when the group is admitted,
and $v_t$ is the policy version immediately before the step that
consumes it.

The target queue depth $Q_{\mathrm b}$ and steady-state mean scheduling
staleness satisfy
\begin{align}
  Q_{\mathrm b} &= (\varphi+\mu)B, \\
  \mathbb{E}[d] &= \varphi + \frac{\mu-1}{2}.
  \label{eq:mean-version-diff}
\end{align}
The $\varphi$ term accounts for extra dispatched groups, while
$\tfrac{\mu-1}{2}$ averages the training-step positions within a burst,
indexed from zero.

\paragraph{GPU placement policy and utilization.}

{\emph{Async}} assigns fixed GPU pools to rollout and training,
allowing the two stages to overlap \citep{areal,streamrl}.
Long rollout latency requires high concurrency to sustain throughput,
but increasing dispatch depth also raises scheduling staleness
(Equation~\ref{eq:mean-version-diff}).
Limiting dispatch depth can therefore leave training GPUs waiting for
data, while the fixed split prevents training from using spare rollout
capacity as executions finish (Figure~\ref{fig:gpu-utilization},
{Async}).

{\emph{Colocate}} alternates a shared GPU pool between rollout and
training, as supported by HybridFlow and ReaL
\citep{hybridflow,real}.
Because the pool switches as a unit, training must wait for the rollout
tail even when groups are ready and GPUs are idle
(Figure~\ref{fig:gpu-utilization}, {Colocate}).
More steps per burst amortize switching and weight-publication costs,
but increase scheduling staleness for later steps.

\subsection{From Harness Executions to Training Data}
\label{sec:background-data}

At \xlong{}-horizons, context management and sub-agent delegation routinely
break the append-only history of a simple agent loop~\citep{react,sweagent}.
Harnesses compact histories as context windows fill and maintain
separate contexts for sub-agents.

The harnesses used for these executions are also used in deployment, and
the most capable, such as Claude Code and Codex, are closed and evolve
rapidly; reimplementing their control flow inside an RL framework is
impractical and would train the policy under a harness different from
the one it is deployed with. This motivates the black-box interface of
Section~\ref{sec:background-workload}, which exposes only model calls
and task-level evaluation results. Using these records for group-relative RL requires valid
outcome rewards and explicit choices of training targets and loss weights.

\paragraph{Execution structure.}
An execution can contain multiple \emph{trajectories}. A trajectory
is a linear path of {model calls} in which each successive call
extends the preceding {call's} recorded context and response.
Compaction and other context edits can rewrite an existing history
  \citep{agentlightning_v1}. 
{Calls} may remain causally related across these changes without
forming one append-only history. A single agent or role can therefore
contribute multiple trajectories.
Concatenating {requests} across these boundaries as one history can give
responses contexts they did not have during generation.
Across hundreds of interactions, repeated context changes produce varying numbers of trajectories with different lengths.

\paragraph{Reward validity.}
An execution's reward can fail to support training in two
distinct ways. First, an execution may reach its deadline before
completing the task \citep{swemarathon}: its captured trajectories
describe a partial attempt, and assigning all partial attempts the
same failure reward obscures differences in task progress. Second,
infrastructure or evaluator failures may leave no valid assessment at
all; a missing assessment is distinct from a valid zero reward and
cannot be filled in as one. Even after every execution in a group has
terminated, the group may therefore lack the $N$ valid sample rewards
needed for advantage estimation.

\paragraph{Training semantics.}
A task-level outcome reward evaluates an execution of the original task.
Figure~\ref{fig:data-hierarchy} illustrates the varying trajectory
counts: each query has $N$ executions, while execution $i$ contributes
$J_i$ selected trajectories. A flat average of trajectory losses gives
executions with more selected paths greater aggregate weight
\citep{agentlightning_v1}.
Within an execution, the same averaging can give many short auxiliary
paths more aggregate weight than a long main-agent path. A shared
response can likewise receive extra weight if each copy contributes
to the loss. The RL objective must therefore specify which roles, trajectories, and tokens contribute to the loss, together with the normalization and weighting rules.

%% file: content/horizongyre_overview.tex
 \section{\sys{} Overview}
  \label{sec:overview}

  \begin{figure}[t]
  \centering
  {\includegraphics[width=\columnwidth]{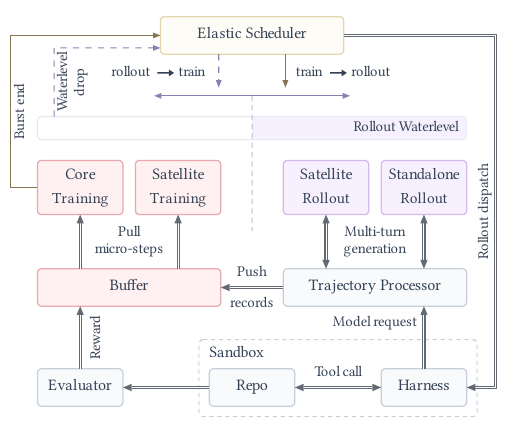}}
  \caption{Elastic scheduling and data flow in \sys{}.
  The sandbox hosts the harness and repository, linked by tool calls;
  Solid and dashed control
  links denote burst-boundary actions and {waterlevel} feedback,
  respectively. Double lines show data flow; the buffer combines trajectory
  storage and stream packing.}
  \label{fig:framework}
  
\end{figure}

{\sys{} supports online RL through unmodified black-box agent
harnesses at \xlong{}-horizons. Its architecture consists of two
cooperating parts (Figure~\ref{fig:framework}).
The elastic scheduler (Section~\ref{sec:elastic}) reallocates
GPUs between rollout and training while allowing live harness
executions to continue.
The trajectory processor (Section~\ref{sec:structured-trajectory})
constructs training samples from recorded {model calls},
retaining execution records independently of the samples admitted
for optimization.}

\paragraph{Elastic scheduler.} GPU resources comprise elastic cells
and an optional standalone rollout pool. Cells share a training parallel
layout; each can serve rollouts or host a complete {training actor} replica.
The black-box proxy, as part of the trajectory processor forwards model calls to rollout engines and
reroutes them during role changes, while harnesses and workspaces remain
outside GPU workers. The elastic scheduler dispatches rollout work at
startup and burst boundaries and tracks queued and running executions as the
rollout waterlevel. 
As this level falls,
cells move to training when the
remaining rollout capacity covers outstanding work. 
The core cell holds the {authoritative weights} and
alone maintains optimizer state. {Satellite cells pull its
parameter snapshot and can join the ongoing batch before the rollout tail
completes.}

\paragraph{Trajectory processor.} The proxy records exact
{model-request} tokens and behavior log-probabilities in trajectory
trees that share common prefixes and preserve each output's original
conditioning context. {The evaluator scores each execution's preserved
workspace, including assessable partial progress after timeout, and sends
the reward and its validity status to the buffer.} At training admission,
\sys{} computes group-relative advantages from the original execution
rewards and selects at most $J_{\max}$ trajectories per execution by role
priority. Selected trajectories inherit their execution's advantage. Only
policy-generated tokens are eligible targets, and shared targets contribute
once. {Averaging losses over selected trainable tokens within each
execution, then over executions in the batch, prevents additional
trajectories from increasing an execution's aggregate weight.}

{The two parts connect through a shared stream of training
micro-steps, allowing the core to begin a batch before all samples are
materialized.} Execution termination lowers the scheduling {waterlevel};
a group becomes ready for training only after all its executions have closed
and their designated outcomes are valid. {{The buffer in
Figure~\ref{fig:framework} stores execution records and associated rewards
and packs selected paths from ready groups on demand into micro-steps
carrying loss masks and weights.}} {Training cells consume
this stream, and the core performs one optimizer update per batch.}

Sections~\ref{sec:elastic} and~\ref{sec:structured-trajectory} detail the
elastic scheduling and trajectory processing planes, respectively.

%% file: content/elastic_scheduling_plane.tex
\section{Elastic Scheduler}
\label{sec:elastic}

\begin{figure*}[t]
  \centering
  {\includegraphics[width=\textwidth]{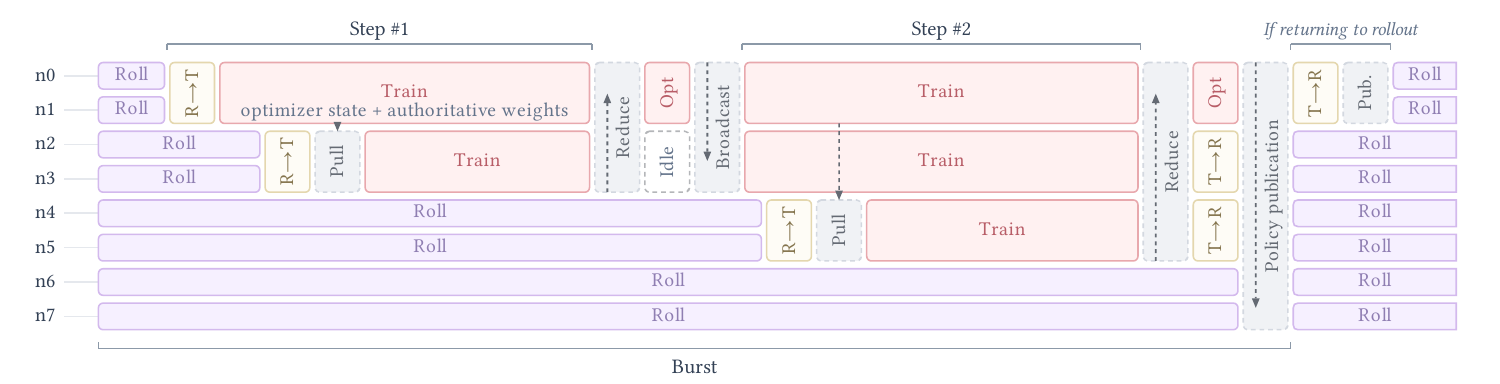}}
  \caption{Illustrative timeline of a two-step burst.
  The core (n0--n1) performs both optimizer updates.
  Satellite cells join ongoing steps after pulling its current
  parameters; n6--n7 stay in rollout during this burst.
  Gradient reduction precedes each update; a parameter broadcast
  follows the nonfinal update, and rollout policy publication follows
  the final update. The core's final role transition and local
  publication (Pub.) occur only when it returns to rollout.}
  \label{fig:execution}
\end{figure*}

The elastic scheduling plane coordinates role
transitions and streaming training to use available GPUs while
serving ongoing rollouts.

\subsection{Topology-Aware Resource Organization}

\paragraph{Cell layout and roles.}

\sys{} uses Megatron{~\citep{megatron}} for training and
SGLang{~\citep{sglang}} for rollout. Its $K$ \emph{elastic cells}
share a training parallel layout; an optional \emph{standalone pool}
remains in rollout.
We favor small cells for fine-grained resource reallocation, subject to
parallelism and topology constraints. Tensor, pipeline, expert, and context
parallel groups remain within each cell, with placement favoring
high-bandwidth interconnects. 

Cells follow a fixed order, with the \emph{core cell} first and
\emph{satellite cells} thereafter.
The core holds the authoritative weights and is the sole
optimizer owner.

\paragraph{Rollout capacity.}

Rollout capacity is measured as the number of concurrent rollout executions
allowed by \sys{}. Let $C_e$ and $C_s\ge0$ denote the
capacities of each elastic cell in rollout mode and the standalone pool, respectively, with $C_s=0$
when the standalone pool is absent. With $k(t)$ cells assigned to training,
the available rollout capacity is
\begin{equation}
    C_R(t)=C_s+\bigl(K-k(t)\bigr)C_e.
    \label{eq:rollout-capacity}
\end{equation}

\subsection{Boundary Dispatch and Waterlevel Control}
\label{sec:burst-dispatch}

\paragraph{Boundary dispatch.}

To implement boundary dispatch (Section~\ref{sec:background-scheduling}),
the scheduler tracks three cumulative execution counts: $d(t)$
increases when executions are dispatched into the rollout queue,
$p(t)$ increases when they launch, and $f(t)$ increases when they
terminate. Thus
$d(t)-p(t)$ counts queued executions, $p(t)-f(t)$ counts running
executions, and $0\le f(t)\le p(t)\le d(t)$.
Each launched execution increments $f(t)$ exactly once on completion,
failure, timeout, or cancellation, including surplus executions under
oversampling. Model-request retries remain within the same execution and
leave all counters unchanged. Queued entries record the published
policy version at dispatch; harnesses and sandboxes are created at launch.

In execution units, the initial dispatch is $d(0)=(\varphi+\mu)NB$,
with $\mu NB$ new executions dispatched after each burst's weight
publication.
Whenever there is spare rollout capacity {(Equation~\ref{eq:rollout-capacity})}
and engines are ready to serve, the scheduler immediately fills available
slots by launching
{\begin{equation}
    \min\bigl\{d(t)-p(t),\ C_R(t)-\bigl(p(t)-f(t)\bigr)\bigr\}
    \label{eq:rollout-launch}
\end{equation}}%
queued executions and leaving any excess queued. {The launches in
Equation~\ref{eq:rollout-launch} use the existing quota, so $d(t)$ remains
fixed within the burst.}

\paragraph{{Waterlevel} control and training supply.}

The {waterlevel} $w(t)=d(t)-f(t)$ is nonincreasing within a
burst. The next cell can join training only if the remaining
rollout capacity covers these outstanding executions:
{\begin{equation}
    w(t)\le C_R(t)-C_e.
    \label{eq:cell-admission}
\end{equation}}%
{Equation~\ref{eq:cell-admission} is paired with a supply check that
counts ready, unassigned groups, capped by the number of groups the
current batch still needs.} This count must cover at least
a fraction $\rho_{\min}\in(0,1]$ of the batch's $B$ groups continuously
for $\tau$ before a transition begins. The timer resets after each
transition and at each batch start. Cells transition serially, core
first and then satellites by index, and remain in training throughout
the burst, forming a growing prefix.

\paragraph{Burst-End Weight Sync and Core Retention.}

After the burst's final optimizer update, all training satellites
return to rollout and restore their SGLang engines. The core then
publishes the updated weights to all satellite and standalone rollout
engines, preparing them for the next burst.

After synchronization, the core remains in training only if at least
$\rho_{\min}B$ ready, unassigned groups are available for the next batch
and the remaining rollout capacity covers outstanding work plus the
pending $\mu NB$ executions:
\begin{equation}
  w(t)+\mu NB\le C_s+(K-1)C_e.
  \label{eq:core-retention}
\end{equation}

If either the supply check or Equation~\ref{eq:core-retention} fails,
the core returns to rollout and reshards its parameters into the local
SGLang (Figure~\ref{fig:execution}).

\subsection{Harness-Preserving Role Transitions}
\label{sec:elastic-execution}

\paragraph{Request rerouting.}

During a role transition, the proxy reroutes model calls
while the harness, workspace, and tool execution state remain in place.
The source cell first stops accepting new requests. The scheduler
reserves destination capacity for affected executions and updates their
routes, then cancels outstanding requests on the source engines.
The proxy transparently retries these requests at the destination
engines. The scheduler prioritizes available standalone engines,
followed by engines in rollout cells in descending cell order.
This ordering favors destinations expected to remain available longer,
reducing repeated rerouting of the same execution within a burst.

\paragraph{KV-cache migration.}

KV-cache migration can reduce prefix recomputation during
request rerouting. Before a request is canceled on the source engine, the
source pins the corresponding cache blocks. Once the request is canceled
and cache writes have stopped, the destination can read these blocks
through RDMA. The source releases the cache only after the destination
acknowledges receipt and takeover.

\paragraph{Role-state management.}

After request rerouting and cache handoff are complete, the cell
offloads rollout parameters and releases KV caches and temporary buffers
to reclaim GPU memory, then restores training parameters and buffers.
Backend processes and communication groups remain alive across role
changes.

\subsection{Streaming Training with Dynamic Membership}

\paragraph{Centralized dynamic data parallelism.}

To let satellites join an ongoing batch, the core exposes
{buffers holding its authoritative weights} through Mooncake's
Transfer Engine during its transition from rollout to training for the
burst's first step, and before each subsequent step.
While training on the current batch is still in progress, satellites
can join in the predefined cell order. Each satellite registers with
the scheduler, pulls the core's weights via RDMA, and begins
training without interrupting the core (Figure~\ref{fig:execution}).
The weights remain fixed during gradient accumulation, so joining
satellites read a consistent snapshot. The controller closes the
joining window before the optimizer update.

{After the entire batch has been processed, the core aggregates and
normalizes gradients from all training cells.} It then applies gradient clipping and
performs a single optimizer update. After each nonfinal update, it
broadcasts the new parameters to the current satellites, which remain
in training for the next step.

For $K$ ordered cells, we prebuild $K-1$ sets of communication groups,
one for each prefix of $2,\ldots,K$ cells. Within each set, groups
connect ranks holding corresponding model shards across cells.
Each collective uses the set matching its participating prefix.

\paragraph{Streaming training.}
\label{sec:elastic-streaming}

Streaming training uses two buffers, shown as the Buffer in
Figure~\ref{fig:framework}. The \emph{trajectory buffer} stores execution
records and associated rewards by rollout group. The \emph{stream
packing buffer} prepares training inputs on demand. When its buffered
trajectories run low, it pulls a ready group from the trajectory
buffer, prioritizing the oldest policy version. It materializes the
group's selected trajectories into token rows and reorders the rows
for efficient packing.

The packing buffer emits training inputs in small \emph{micro-steps},
each providing data for a cell's data-parallel ranks. Cells atomically
claim micro-steps from the packing buffer's shared queue as compute
capacity becomes available, receiving disjoint work. Earlier or faster cells
consume more micro-steps, helping participating cells finish the batch
at similar times. The packing buffer selects the batch's $B$ groups
incrementally and closes the stream only after enqueuing all their
micro-steps.

Each cell processes incoming micro-steps with a continuous
one-forward-one-backward (1F1B) pipeline. All stages within a cell use
the same signal for input availability. When data are temporarily
unavailable, the stages complete outstanding communication and pause
at consistent scheduling points, preserving pipeline state. After the
stream closes and the queue is exhausted, the pipeline drains its
remaining work. Each cell that processes micro-steps therefore incurs
one startup and one final drain per batch.

%% file: content/trajectory_processing_plane.tex
 \section{Trajectory Processor}
  \label{sec:structured-trajectory}

   {The trajectory processing plane constructs training samples
  from harness execution records and task outcomes. During rollout,
  \sys{} preserves the exact {model-request} tokens and each
  output's original conditioning context using token-in, token-out (TITO)~\citep{tito}
  (Section~\ref{sec:trajectory-records}). After an execution terminates,
  \sys{} evaluates its preserved work, including assessable partial
  progress after timeout, under the task's scoring rule to obtain
  an execution-level reward
  (Section~\ref{sec:trajectory-partial-score}).}
  {{At training admission, \sys{} selects a bounded set of paths
  per execution from ready groups by role priority and fixes target
  masks so shared targets contribute once
  (Section~\ref{sec:trajectory-sampling}).}}

  \subsection{TITO and Trajectory Trees}
\label{sec:trajectory-records}
\label{sec:trajectory-tito}

\paragraph{Request records with TITO}
\sys{} implements TITO in the
black-box proxy (Section~\ref{sec:elastic-execution}) to maintain
a record of each model call.
Each record preserves the exact input and output token IDs, the output tokens' behavior log-probabilities, and associated metadata.

To maintain these records across harness message processing,
the proxy also caches original tool-call payloads by call ID.
Before matching a subsequent request's history, it restores these
payloads while preserving the incoming IDs for tool-response
association. This avoids spurious history mismatches caused by
tool-call reformatting.
The harness keeps its existing message interface, while training
consumes the recorded tokens and associated metadata.
Each output retains its original conditioning context even if
later compaction removes it from the harness's history.

\paragraph{Prefix sharing through trajectory trees}
\sys{} organizes the input/output tokens and associated metadata
in the TITO request records described above into a trajectory tree
for each execution, following the prefix representation used in
black-box agent training~\citep{clawgym2}.
Each node is built from a request record and stores its incremental
input, recorded output, and associated metadata.
After the tool-call restoration described above, the framework
matches an incoming request's serialized context against existing
paths and reuses the tokens stored in the matched nodes.
Only new input content is encoded, and the new request record is
appended as a node.
Diverging contexts create branches; requests with no matching
prefix attach to the root.
This organization captures the different context paths introduced
by compaction and sub-agent calls
(Figure~\ref{fig:data-hierarchy}(c)), while retaining independently
generated outputs as separate nodes.

Candidate trajectories remain in this shared representation until
selection, avoiding full-history expansion for every path.
Selected paths are materialized from the stored token segments
and their aligned metadata.

\subsection{Partial Scoring}
  \label{sec:trajectory-outcomes}
  \label{sec:trajectory-partial-score}

  \paragraph{{Scoring preserved work}}
  {At execution closure, \sys{} fixes the committed
  policy-generation request record and identifies any available
  artifact state to assess.}
  \sys{} evaluates the preserved work under the task's scoring rule
  after an execution terminates, including after timeout.
  {An unfinished execution can receive a valid score when
  its completed work is assessable.
  For these assessed outcomes, the data plane does not derive rewards
  from elapsed time, trajectory length, or termination status.}
  {Using executable checks or task-specific agentic
  evaluation~\citep{agent_as_judge}, the evaluator returns a score,
  its validity status, and a reference to the assessed state.}

  \paragraph{Preserving the scope of an assessment}
  An \xlong{} execution may evaluate intermediate artifacts, revise
  them, and abandon branches.
  {Each assessment remains associated with its execution,
  artifact version, and available branch reference; later assessments
  do not overwrite these associations.}
  {A critique used by a subsequent policy request remains
  part of its input, but its evaluator provenance excludes it from
  policy training targets.}
  {The designated task assessment supplies one reward for
  each scored execution; intermediate assessments remain associated
  with the states they evaluate.}

  \paragraph{{Handling failures}}
  {An evaluator failure leaves the assessment unresolved;
  evaluation may be retried against the same preserved state within
  the execution's overall time budget.
  If this budget expires without a valid reward, the recorded
  trajectories are excluded from training.
  An unresolved assessment is distinct from a valid zero score,
  which remains usable for training.}

  {For rare complete failures that leave no assessable work,
  such as image-pull failures, \sys{} inserts a dummy trajectory with
  zero reward as a failure placeholder.}

\subsection{Trajectory Sampling}
\label{sec:trajectory-admission}
\label{sec:trajectory-sampling}

The trajectory tree records all context in the execution, which \sys{}
turns into training data in two steps: it first masks every token that
should not be trained, and then converts the tree into individual
training trajectories.

\paragraph{Trainable tokens are defined by provenance.}
A node's tokens are eligible training targets only if the policy produced
them during rollout. TITO (Section~\ref{sec:trajectory-tito}) records
whether each token was emitted by our SGLang engines or supplied from
outside. System prompts, user prompts, tool results, and
assistant-role content the harness injects rather than samples from the
policy will be masked and carry no training signal: they are retained as
conditioning context only.

\paragraph{Sampling trajectories from the tree.}
Candidate trajectories correspond to the leaves of the trajectory tree:
a leaf together with the prefix path from the root defines one
trajectory. At \xlong{}-horizons a single execution can produce many
candidate trajectories, and their number may vary widely across
executions. Simply retaining all trajectories may lead to a decrease in training efficiency, so \sys{} retains at most $J_{\max}$ training trajectories
per execution and fills this budget in order of expected importance. To
rank the candidates, \sys{} classifies each leaf by the role of the {request}
sequence it terminates. For a Claude Code harness these classes are, for
the main agent and for each sub-agent, its task-driving trajectory and
its summary trajectory, ordered by the priority
{\begin{equation}
\begin{aligned}
\text{main} &> \text{main-summary} \\
            &> \text{sub-agent} > \text{sub-agent-summary}.
\end{aligned}
\label{eq:trajectory-priority}
\end{equation}}%
{The ordering in Equation~\ref{eq:trajectory-priority} follows two
principles.} The main agent's turns are what the
task-level reward most directly credits, whereas sub-agent trajectories
contribute only through delegated subtasks; and within each agent, turns
that advance the task rank above the summary trajectory, which manages
context rather than acting on the task.

Admission starts from the provenance mask above and draws trajectories one at a time. Each step takes the
highest-priority class that still has a leaf with unmasked trainable
tokens, and draws one leaf from it with probability proportional to its
unmasked count. The drawn trajectory's targets are fixed to the tokens
still unmasked at this draw; \sys{} then masks its entire root-to-leaf
path in the tree. This keeps a shared prefix from being trained more than once and leaves
already-fixed targets untouched, while the lowered counts reweight the
remaining leaves for the next draw. Admission repeats until no unmasked
trainable token remains or $J_{\max}$ trajectories are admitted, yielding
$J_i\le J_{\max}$ trajectories from execution $i$
(Figure~\ref{fig:data-hierarchy}), each inheriting the execution's
group-relative advantage.

{The admitted trajectories retain their execution identities,
recorded contexts and behavior log-probabilities, and fixed target
masks. Materialization and packing preserve these annotations for the
execution-level objective.}

%% file: content/experimental_results.tex
\section{Experimental Results}
\label{sec:experiments}

\begin{figure*}[!t]
  \centering
  \includegraphics[width=\textwidth]{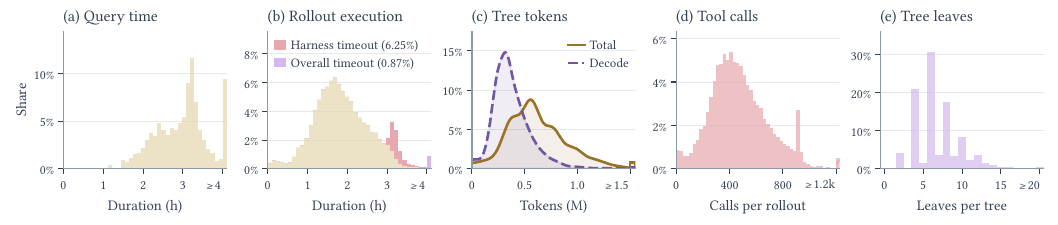}
  \caption{{NL2RepoBench} execution distributions with Qwen~3.6 122B
    at $\mathbb{E}[d]=1.5$ over steps 1--48.
    {(a) Query durations, measured as the longest of each query's
    16 rollout execution durations.}
    (b) Rollout execution durations.
    {(c) Total (solid) counts input and output tokens per rollout
    tree; Decode (dashed) counts generated output tokens. Both count
    shared prefixes once. The smoothed curves show shares per 50k tokens.}
    (d) Tool calls per rollout execution.
    (e) Leaves per rollout tree.
    In (b), pink and purple mark harness and overall timeouts, respectively;
    executions with both flags appear only in purple.
    Final bins labeled $\geq$ collect values at or above the indicated thresholds.}
  \label{fig:nl2repo-air-distribution}
\end{figure*}

\begin{figure*}[t]
  \centering
  \includegraphics[width=\textwidth]{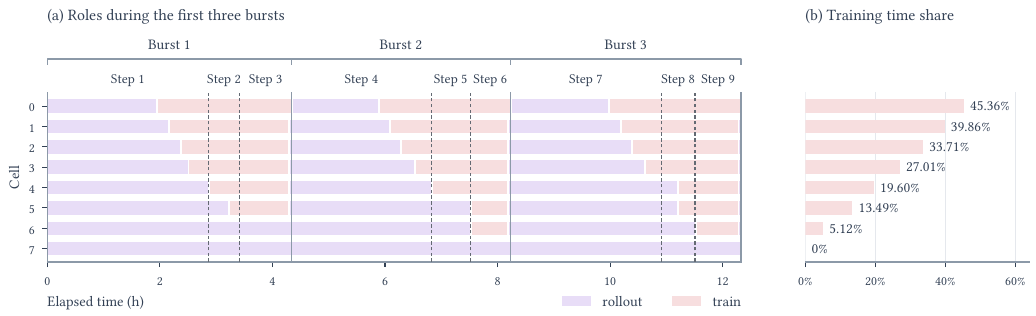}
  \caption{Execution profile of \sys{} on {NL2RepoBench} with Qwen~3.6 122B
    at $\mathbb{E}[d]=1.5$.
    (a) Cell roles during the first three bursts (steps 1--9).
    (b) Each cell's recorded training time as a fraction of the full
    observation window. {Cell 7 remains in rollout throughout this
    interval.}}
  \label{fig:nl2repo-runtime}
\end{figure*}

\begin{figure}[!t]
  \centering
  \includegraphics[width=\linewidth]{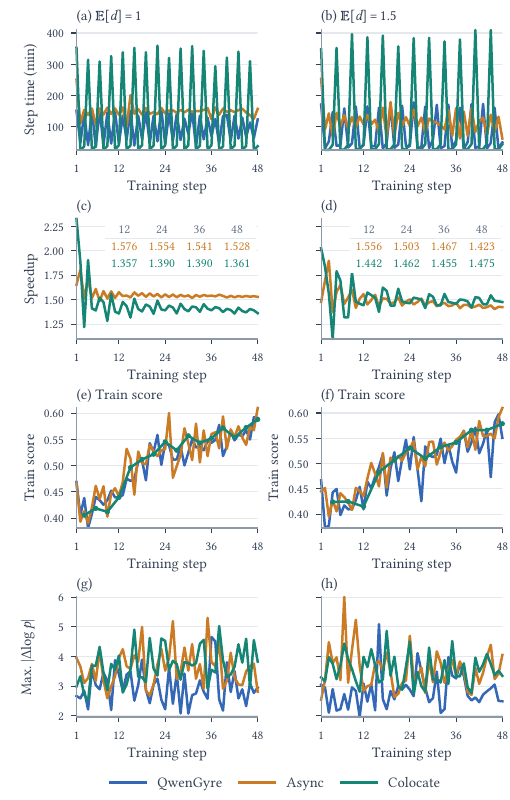}
  \caption{{NL2RepoBench} results with Qwen~3.6 122B at $\mathbb{E}[d]=1$ (left) and
    $\mathbb{E}[d]=1.5$ (right). Speedup is cumulative baseline time divided by
    cumulative \sys{} time. In (c,d), in-plot tables give speedups at
    selected checkpoints: Async in the upper row and Colocate
    in the lower row, with matching curve colors.
    Max.\ $|\Delta\log p|$ is the maximum absolute difference between
    current-policy and recorded rollout-policy log-probabilities
    over unmasked response tokens.}
  \label{fig:nl2repo-overview}
\end{figure}

\begin{figure}[!t]
  \centering
  {\includegraphics[width=\linewidth]{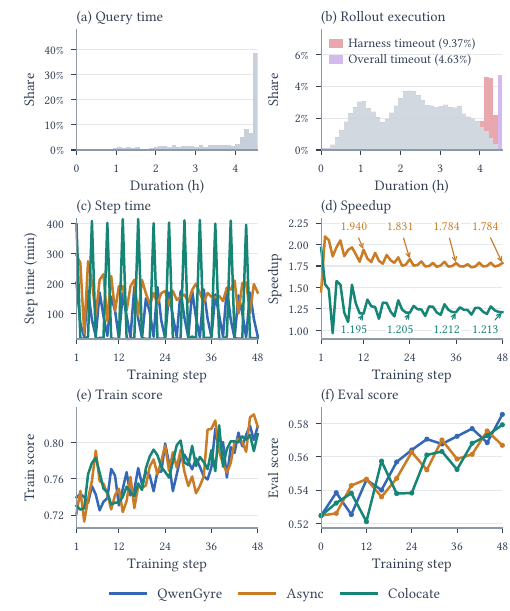}}
  \caption{{NL2RepoBench} with Qwen~3.8 2.4T. (a) Measured query durations,
    using the longest rollout execution wall time per query.
    (b) Measured rollout execution durations.
    The final bin collects durations $\geq 4.45$~h.
    In (b), pink and purple mark harness and overall timeouts,
    respectively.
    {(c,d) Step times and cumulative speedups over 48 steps.}
    Arrows in (d) label speedups at selected checkpoints in matching colors.
    {(e,f) Training and evaluation scores.}}
  \label{fig:nl2repo-max}
\end{figure}

\subsection{Experimental Setup}

\begingroup
\setlength{\emergencystretch}{2em}
\paragraph{Tasks.}
We evaluate \sys{} on three agentic workloads: {NL2RepoBench~\citep{nl2repobench}}, DeepSWE{~\citep{deepswe}}, and {TerminalBench~\citep{terminalbench}}.
We construct training data internally for all three workloads.
{For benchmark evaluation, we follow each benchmark's prescribed evaluation protocol.}
{NL2RepoBench} covers the \xlong{} setting of building software repositories
from natural-language specifications;
during training, we use a separate Qwen~3.7-Max evaluator to score the
generated code using independently written tests.
DeepSWE covers software engineering in existing repositories;
during training, we use repository tests for evaluation.
{TerminalBench} comprises multi-turn terminal tasks;
during training, we use task-specific verifiers for evaluation.
All workloads use Claude Code~2.1.220{~\citep{claude_code}} to execute shell commands and
manipulate files in containerized environments.
\par
\endgroup

\paragraph{Models.}
We use Qwen~3.6 122B and Qwen~3.8 2.4T.
For Qwen~3.6 122B, the {NL2RepoBench} experiments use 32 nodes, organized into
eight cells of four nodes each by default, with an inference concurrency limit of
144 per cell. {The cell-granularity ablation also evaluates four
8-node cells under the same 32-node budget.}
DeepSWE and {TerminalBench} follow the same protocol, using 24 nodes
organized into six cells with an inference concurrency limit of 192 per cell.
For Qwen~3.8 2.4T, the {NL2RepoBench} traces use four cells of 48 nodes each,
with a total inference concurrency limit of 384.
Unless otherwise stated, we use no dedicated standalone rollout nodes.

\paragraph{RL algorithm.}
\label{sec:rl-algorithm}
{All methods use our adaptation of Group Sequence Policy
  Optimization (GSPO)~\citep{gspo}, combining token-level importance
  weighting with execution-level clipping. Token importance weights
  are capped at 5, and the sequence-level clipping bounds are
  $[0.995,1.005]$. Each group contains 16 independent rollout executions;
  each update consumes 32 groups for NL2RepoBench or 48 for DeepSWE
  and TerminalBench. The default trajectory cap is $J_{\max}=5$ per rollout
  execution. All methods use Adam~\citep{adam} with a constant
  learning rate of $10^{-6}$.}

\paragraph{Baselines.}
We compare \sys{} with two scheduling baselines.
\emph{Async} assigns separate, fixed GPU pools to rollout
and training, allowing training steps to consume ready rollout groups
while other executions continue. It uses boundary dispatch with a
target queue depth, refilling consumed groups' slots after weight
publication.
\emph{Colocate} alternates a shared GPU pool between rollout and
training. Once the admitted rollout executions finish, the entire pool
switches to training and processes the completed rollout groups in a
burst of consecutive training steps.

{We compare methods at common target values of mean scheduling
staleness $\mathbb{E}[d]$, defined in
Section~\ref{sec:background-scheduling}. All three methods use boundary
dispatch; we choose the extra-dispatch parameter $\varphi$ and the
number of training steps per burst $\mu$ using
Equation~\ref{eq:mean-version-diff}. The original Colocate cannot
dispatch extra groups without interrupting the harness, so it uses
$\varphi=0$. For example, at $\mathbb{E}[d]=1.5$, Async can use
$\mu=1$ with $\varphi=1.5$, Colocate uses $\mu=4$ with $\varphi=0$,
and \sys{} can use $\mu=3$ with $\varphi=0.5$.}

\subsection{End-to-End Performance}
\label{sec:end-to-end}

\paragraph{{NL2RepoBench} on Qwen~3.6 122B}
Figure~\ref{fig:nl2repo-air-distribution} shows mean durations of 1.93 hours
per rollout execution and 2.96 hours per query; 9.51\% of queries
take at least four hours.
Figure~\ref{fig:nl2repo-air-distribution}(b) shows harness and overall
timeout rates of 6.25\% and 0.87\%, respectively.
Harness timeouts indicate that task execution has timed out; the results
are still scored and used for training.
{Failed evaluations can be retried within the overall execution
time budget. An overall timeout exhausts this budget before a valid
reward is obtained, so the execution is discarded from training.}

Figure~\ref{fig:nl2repo-overview}(a--d) shows that \sys{} achieves
$1.42$--$1.53\times$ speedups over Async and $1.36$--$1.47\times$
over Colocate after 48 training steps at both {mean scheduling staleness targets}.
Figure~\ref{fig:nl2repo-overview}(e,f) shows
comparable training-score curves across methods, indicating that faster
training is achieved while maintaining task performance.

The role timeline in Figure~\ref{fig:nl2repo-runtime}(a) shows how \sys{}
dynamically allocates resources between training and inference.
Cells enter training at different points while other cells remain in rollout.
Figure~\ref{fig:nl2repo-runtime}(b) shows how training time shares vary
across cells over the full observation window.
This staggered allocation allows training to begin with a subset of
cells and subsequently expand. It uses resources that a fixed
Async partition leaves assigned to rollout and avoids requiring
Colocate's pool-wide phase transition before training can proceed.
{The average switching times are 8.52~s for rollout to training
 and 3.46~s for training to rollout.
Both are negligible compared with rollout executions lasting hours.}

\paragraph{{NL2RepoBench} on Qwen~3.8 2.4T}
On {NL2RepoBench}, Qwen~3.8 2.4T has a higher timeout rate than
Qwen~3.6 122B, even though we extend the timeout limit.
Figure~\ref{fig:nl2repo-max}(a,b) shows the query and rollout execution duration
distributions. Among queries with complete duration records,
61.1\% take at least four hours, compared with 9.51\% for
Qwen~3.6 122B.
In 38.2\% of these queries, the longest rollout execution terminates with
an overall timeout. 

{As shown in Figure~\ref{fig:nl2repo-max}(c,d), \sys{} completes
48 training steps in 75.42 hours, compared with 134.55 hours for Async
and 91.47 hours for Colocate, yielding end-to-end speedups of
$1.78\times$ and $1.21\times$, respectively.} The concentration of query durations
near the timeout limit helps
explain the difference in gains over the two baselines. It suggests
fewer opportunities for \sys{} to reclaim cells early and expand
training gradually, consistent with the smaller gain over Colocate,
which also shares its resource pool between training and rollout.
Async's fixed partition leaves resources reserved for training
unavailable to rollout, even when training is idle. By dynamically
allocating these resources to rollout, \sys{} can increase inference
capacity during long executions, yielding a larger gain over Async.

{Figure~\ref{fig:nl2repo-max}(e) shows that \sys{}'s training
scores match those of Async and Colocate. Its evaluation passrate
improves from 52.48\% to 58.54\% over 48 training steps
(Figure~\ref{fig:nl2repo-max}(f)).}

\paragraph{DeepSWE and {TerminalBench} on Qwen~3.6 122B}
\label{sec:swe-tb}
DeepSWE and {TerminalBench} share the same end-to-end evaluation protocol: Qwen~3.6
122B, 24 training steps, and all three methods (\sys{}, Async, and
Colocate) at target {mean scheduling staleness values $\mathbb{E}[d]=1$ and $1.5$}.
Table~\ref{tab:swe-tb} reports the resulting cumulative speedups.

\begin{table*}[t]
  \centering
  \small
  \caption{Scheduling ablations over the first 12 {NL2RepoBench} training
    steps with Qwen~3.6 122B at $\mathbb{E}[d]=1.5$.
    A0, C0, and E0 reuse the Async, Colocate,
    and \sys{} base configurations from Section~\ref{sec:end-to-end}.
    Arrows show the evolution paths; {E1 and E2 branch independently
    from E0.}
    Switch is -- for fixed partitions, Global for switching all nodes
    together, Coarse for switching the {Colocate} pool together,
    and Fine for switching cells independently. Standalone nodes remain
    in rollout. Allocation lists training/rollout (T/R) or
    {Colocate}/standalone (C/S) node counts; cell entries give
    cells $\times$ nodes per cell. Total is the time for all 12 steps
    in hours; Time / E0 normalizes it to E0. Lower is better.}
  \label{tab:ablation-scheduling}
  \begingroup
  \definecolor{ablationA}{HTML}{CB7B24}
  \definecolor{ablationC}{HTML}{158575}
  \definecolor{ablationE}{HTML}{3567B7}
  \setlength{\tabcolsep}{3pt}
  \renewcommand{\arraystretch}{1.15}
  \begin{tabular*}{\textwidth}{@{\extracolsep{\fill}}lclccclrr@{}}
  \toprule
  ID & & Evolution & Stream & Switch & $({\varphi},\mu)$ & Allocation & Total (h) & Time / E0 \\
  \midrule
  A0 & \makebox[34pt][l]{\hspace*{5pt}\makebox[0pt][c]{\tikz[remember picture,baseline=-0.5ex] \node[circle,draw=ablationA,fill=white,line width=0.6pt,inner sep=1.3pt] (ablation-A0) {};}} & Original Async & Disabled & -- & $(1.5,1)$ & 16 T / 16 R & 25.08 & 1.556 \\
  A1 & \makebox[34pt][l]{\hspace*{1pt}\makebox[0pt][c]{\tikz[remember picture,baseline=-0.5ex] \node[circle,fill=ablationA,inner sep=1.3pt] (ablation-A1) {};}} & {A0 with adjusted partition} & Disabled & -- & $(1.5,1)$ & 8 T / 24 R & 27.01 & 1.676 \\
  A2 & \makebox[34pt][l]{\hspace*{5pt}\makebox[0pt][c]{\tikz[remember picture,baseline=-0.5ex] \node[circle,fill=ablationA,inner sep=1.3pt] (ablation-A2) {};}} & {A0 + multi-step burst} & Disabled & -- & $(0.5,3)$ & 16 T / 16 R & 32.58 & 2.021 \\
  A3 & \makebox[34pt][l]{\hspace*{11pt}\makebox[0pt][c]{\tikz[remember picture,baseline=-0.5ex] \node[circle,fill=ablationA,inner sep=1.3pt] (ablation-A3) {};}} & A0 + streaming & Enabled & -- & $(1.5,1)$ & 16 T / 16 R & 22.18 & 1.376 \\
  \midrule
  C0 & \makebox[34pt][l]{\hspace*{30pt}\makebox[0pt][c]{\tikz[remember picture,baseline=-0.5ex] \node[circle,draw=ablationC,fill=white,line width=0.6pt,inner sep=1.3pt] (ablation-C0) {};}} & Original Colocate & Disabled & Global & $(0,4)$ & 32 C / 0 S & 23.23 & 1.442 \\
  C1 & \makebox[34pt][l]{\hspace*{30pt}\makebox[0pt][c]{\tikz[remember picture,baseline=-0.5ex] \node[circle,fill=ablationC,inner sep=1.3pt] (ablation-C1) {};}} & {C0 + standalone rollout nodes} & Disabled & Coarse & $(0.5,3)$ & 28 C / 4 S & {22.56} & {1.400} \\
  C2 & \makebox[34pt][l]{\hspace*{26pt}\makebox[0pt][c]{\tikz[remember picture,baseline=-0.5ex] \node[circle,fill=ablationC,inner sep=1.3pt] (ablation-C2) {};}} & A3 + role switching; C1 + streaming & Enabled & Coarse & $(0.5,3)$ & 28 C / 4 S & 19.23 & 1.193 \\
  C3 & \makebox[34pt][l]{\hspace*{32pt}\makebox[0pt][c]{\tikz[remember picture,baseline=-0.5ex] \node[circle,fill=ablationC,inner sep=1.3pt] (ablation-C3) {};}} & {C2 with adjusted allocation} & Enabled & Coarse & $(0.5,3)$ & 16 C / 16 S & 20.69 & 1.284 \\
  \midrule
  E0 & \makebox[34pt][l]{\hspace*{19pt}\makebox[0pt][c]{\tikz[remember picture,baseline=-0.5ex] \node[star,star points=5,star point ratio=2.2,fill=ablationE,draw=ablationE,line width=0.2pt,inner sep=0pt,minimum size=6.5pt] (ablation-E0) {};}} & {C2 + fine-grained elastic allocation} & Enabled & Fine & $(0.5,3)$ & 8 cells $\times$ 4 & 16.12 & 1.000 \\
  E1 & \makebox[34pt][l]{\hspace*{14pt}\makebox[0pt][c]{\tikz[remember picture,baseline=-0.5ex] \node[circle,fill=ablationE,inner sep=1.3pt] (ablation-E1) {};}} & {E0 with one step per burst} & Enabled & Fine & $(1.5,1)$ & 8 cells $\times$ 4 & 17.87 & 1.109 \\
  E2 & \makebox[34pt][l]{\hspace*{24pt}\makebox[0pt][c]{\tikz[remember picture,baseline=-0.5ex] \node[circle,fill=ablationE,inner sep=1.3pt] (ablation-E2) {};}} & {E0 with four 8-node cells} & Enabled & Fine & $(0.5,3)$ & 4 cells $\times$ 8 & {16.47} & {1.022} \\
  \bottomrule
  \end{tabular*}%
  \begin{tikzpicture}[remember picture,overlay]
    \tikzset{evolution/.style={-{Stealth[length=2.5pt,width=2.2pt]},
      line width=0.55pt,shorten <=0.4pt,shorten >=0.4pt}}
    \draw[evolution,draw=ablationA] (ablation-A0) -- (ablation-A1);
    \draw[evolution,draw=ablationA] (ablation-A0) -- (ablation-A2);
    \draw[evolution,draw=ablationA] (ablation-A0) -- (ablation-A3);
    \draw[evolution,draw=ablationA] (ablation-A3) -- (ablation-C2);
    \draw[evolution,draw=ablationC] (ablation-C0) -- (ablation-C1);
    \draw[evolution,draw=ablationC] (ablation-C1) -- (ablation-C2);
    \draw[evolution,draw=ablationC] (ablation-C2) -- (ablation-C3);
    \draw[evolution,draw=ablationC] (ablation-C2) -- (ablation-E0);
    \draw[evolution,draw=ablationE] (ablation-E0) -- (ablation-E1);
    \draw[evolution,draw=ablationE] (ablation-E0) -- (ablation-E2);
  \end{tikzpicture}
  \endgroup
\end{table*}

\begin{table}[t]
  \centering
  \small
  \caption{24-step speedup of \sys{} over each baseline with
    Qwen~3.6 122B. {The $(\varphi,\mu)$ column gives \sys{}'s
    extra dispatch in batch units and training steps per burst.
    Async uses $(\varphi,\mu)=(\mathbb{E}[d],1)$; Colocate uses
    $(\varphi,\mu)=(0,2\mathbb{E}[d]+1)$, following
    Equation~\ref{eq:mean-version-diff}.}}
  \label{tab:swe-tb}
  \begin{tabular*}{\columnwidth}{@{\extracolsep{\fill}}lccrr@{}}
  \toprule
  Dataset & $\mathbb{E}[d]$ & $({\varphi},\mu)$ & vs. Async & vs. Colocate \\
  \midrule
  DeepSWE & 1 & $(0.5,2)$ & 1.569 & 1.610 \\
  DeepSWE & 1.5 & $(0.5,3)$ & 1.572 & 1.824 \\
  \midrule
  {TerminalBench} & 1 & $(0.5,2)$ & 1.383 & 1.809 \\
  {TerminalBench} & 1.5 & $(0.5,3)$ & 1.430 & 1.849 \\
  \bottomrule
  \end{tabular*}
\end{table}

Across the reported DeepSWE and {TerminalBench} configurations, \sys{} provides
consistent end-to-end speedups over both Async and Colocate.

\subsection{Ablation Study}
\label{sec:ablation}

Table~\ref{tab:ablation-scheduling} traces two paths toward \sys{}:
adding streaming training to Async and adding standalone rollout
capacity to Colocate.

\paragraph{Adding streaming to Async.}
Reallocating nodes from training to rollout (A0 to A1) increases
end-to-end time by 7.7\%. Although training nodes wait for rollout data,
executions finish in bursts; the smaller training pool consumes these
bursts more slowly, delaying weight publication and replacement
rollouts under boundary dispatch. Longer training bursts (A2) increase
time by 29.9\% relative to A0 because the fixed rollout pool cannot
supply them efficiently. Streaming (A3) instead reduces time by 11.6\%
by starting training on completed queries before a full batch is ready.
However, streaming captures only part of the potential gain.
The fundamental limitation remains Async's fixed resource
partition:
it cannot reallocate idle GPUs between training and rollout as
their demands change.

\paragraph{Adding standalone rollout nodes to Colocate.}
while the colocated pool switches to training, preserving
harness
state. Although extra dispatch mitigates the impact of rollout
stragglers, fewer training steps are completed per burst,
limiting
the end-to-end improvement over C0 to just 2.9\%.

Adding streaming to C1 (C2) reduces time by 14.8\%. However,
all 28
colocated nodes must still switch together, so streaming
training
can begin only late in the burst. Increasing standalone
capacity
(C3) shrinks the pool that must switch together, but also
leaves
fewer nodes for training. The resulting longer training phase
delays burst completion and the next dispatch, increasing
end-to-end time by 7.6\% relative to C2.
The fundamental limitation is the coarse switching granularity,
which prevents training from fully using the idle capacity that
emerges as colocated rollouts finish at different times.

\paragraph{Combining streaming and fine-grained allocation.}
E0 builds on C2 by replacing its fixed split of 28 colocated
nodes
and four standalone rollout nodes with eight independently
switching
4-node cells, while retaining streaming and three training
steps per
burst. Training can thus start on a subset of cells and expand
as
data arrives while other cells continue rollout, yielding a
$1.19\times$ speedup over C2.

E1 retains E0's streaming and cell allocation but uses one
training
step per burst, with extra dispatch increased from $0.5$ to
$1.5$
to preserve scheduling staleness. Although E1 takes 10.9\%
longer
than E0, it still achieves a $1.24\times$ speedup over A3,
which
also uses streaming and one training step per burst, and a
$1.08\times$ speedup over C2. These results show that fine-
grained
allocation provides substantial gains even without multi-step
bursts; longer bursts offer an additional benefit.

Streaming makes rollout data available for training earlier,
while fine-grained allocation supplies training resources
as that data arrives.

\paragraph{Cell granularity.}
\label{sec:ablation-cell-granularity}
{E2 builds on E0 by regrouping the same 32 nodes from eight
4-node cells into four 8-node cells, retaining streaming and
$(\varphi,\mu)=(0.5,3)$. It completes the 12 training steps in
16 hours 28 minutes (16.47 hours), only 2.2\% longer than E0's
16.12 hours. In E0, rollout fully hides training, with 20.85\%
slack remaining. E2's coarser cell granularity exposes a small
amount of training on the critical path, while most training
remains overlapped with rollout. This explains the modest impact
of reducing the number of cells from eight to four in this setting.}

\paragraph{Trajectory sampling budget.}
\label{sec:ablation-trajectory-cap}

Sampling only the main trajectory per execution is insufficient:
it yields lower training scores and larger gradient norms over
the first 42 training steps
(Figure~\ref{fig:trajectory-cap-ablation}). Increasing the budget
(Section~\ref{sec:trajectory-sampling}) to five trajectories matches the
scores obtained by sampling all available trajectories, with similarly
small gradient norms. {Over the same window, mean forward--backward
time per step with caps of one and five is 41.4\% and 74.8\% of the
uncapped baseline, respectively. A cap of five therefore preserves
comparable training scores while reducing forward--backward time by
25.2\%.}

\begin{figure}[!htbp]
  \centering
  {\includegraphics[width=\columnwidth]{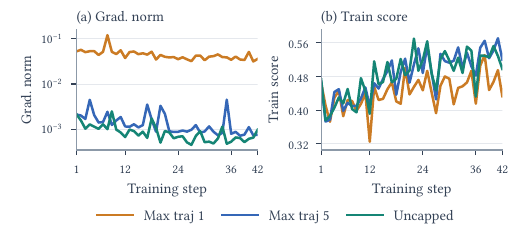}}
  \caption{Trajectory-budget ablation for \sys{} on NL2RepoBench
    with Qwen~3.6 122B {over training steps 1--42}.
    (a) Gradient norm on a logarithmic scale.
    (b) Mean training score. {Max traj 1 samples only the main
    trajectory per execution; Max traj 5 admits up to five trajectories;
    ``Uncapped'' uses all available leaf trajectories.}}
  \label{fig:trajectory-cap-ablation}
\end{figure}

%% file: content/related_work.tex
\section{Related Work}
  \label{sec:related_work}

  \paragraph{Agentic RL execution.}
  General-purpose LLM RL frameworks such as HybridFlow/veRL, OpenRLHF,
  ReaL, and ROLL compose rollout, reference, reward, and policy-update
  workers under different placement strategies
  ~\citep{hybridflow,openrlhf,real,roll}. Agent-oriented runtimes extend
  these abstractions with multi-turn tools, containerized environments,
  asynchronous dispatch, and trajectory-level scheduling
  ~\citep{verltool,ragen,agentrl,skyrl_agent,megaflow,heddle}.

  Synchronous {RL frameworks} preserve explicit policy-update boundaries but expose
  rollout-length skew. Recent work mitigates this tail through scheduling,
  length-aware packing, partial-rollout continuation, and tail isolation
  ~\citep{seer,rollpacker,april,tailsieve}. Asynchronous {RL frameworks} instead
  stream trajectories between disaggregated rollout and trainer pools,
  using bounded staleness, multiple policy versions, or fixed-weight
  trajectory groups to control training semantics
  ~\citep{areal,asyncflow,staleflow,dora,rolloutpipe}. SAO and SPO++ further
  adapt policy optimization to asynchronously arriving agent trajectories
  ~\citep{sao,spopp}.

  \paragraph{Elastic rollout--training scheduling.}
  Recent {RL frameworks} dynamically adjust the resource boundary between rollout
  and training. TideRL adapts rollout and reference execution using
  readiness signals~\citep{tiderl}; BiDiRL allows either side of an
  asynchronous pipeline to borrow idle resources from the other
  ~\citep{bidirl}; and DynaResize reallocates GPUs between rollout and
  training pools using communicator reuse and state staging
  ~\citep{dynaresize}. Libra combines a global resource planner with an
  elastic hybrid pool whose workers can join training as additional
  data-parallel replicas~\citep{libra}. Other {RL frameworks} obtain elasticity
  from complementary phases across RL jobs, online-serving clusters, or
  preemptible rollout resources~\citep{rollmux,rose,rlboost}, while
  AReaL-DTE reduces cross-role policy synchronization through sparse weight
  transfer~\citep{areal_dte}.

  \paragraph{Black-box harnesses and structured trajectories.}
  Harness-native {RL frameworks} connect opaque agent runtimes to policy
  optimization by instrumenting or proxying {model calls}. Agent Lightning
  decouples agent execution from training, while Polar and LEGO-RL preserve
  rollout-time token identity under harness-controlled context processing
  ~\citep{agentlightning,agentlightning_v1,polar,legorl}. ClawGym II
  reconstructs captured {requests} as shared-prefix trees and adapts PPO and
  GRPO so that each shared node contributes training signal only once
  ~\citep{clawgym2}; its retained paths share a reward derived from the
  final workspace evaluation.

  Related work exploits execution structure at other
  layers. BPO, IAPO,
  and MileGPO use branches, dependency graphs, or
  intermediate milestones
  for finer-grained credit
  assignment~\citep{bpo,iapo,milegpo}, whereas
  psRL and related {frameworks} reuse shared
  prefixes
  {after sample construction}
  ~\citep{psrl,treetraining}.
  These techniques address trajectory capture, credit
  assignment, or
  prefix computation individually. \sys{} instead
  couples
  readiness-driven actor reconfiguration with a
  trajectory processor
  {preserving} role, branch, environment-state,
  and evaluator
  provenance throughout \xlong{} rollout collection,
  admission, and
  training materialization.

%% file: content/limitations.tex
\section{Limitations}
\label{sec:limitations}

\paragraph{GPU budget and cell granularity.}
Elastic scheduling needs multiple independently switchable cells
while maintaining enough rollout capacity to serve unfinished
executions. Each cell must accommodate the model's memory and
training parallelism requirements, so a very small GPU budget may
not support this organization. Such deployments are generally
impractical for the resource-intensive \xlong{} workloads targeted
here. Within our evaluated regime, a modest number of cells is
sufficient: under the same 32-node budget, regrouping eight 4-node
cells (E0) into four 8-node cells (E2) increases execution time by
only 2.2\% (Table~\ref{tab:ablation-scheduling}). E0 fully hides
training behind rollout, while E2 exposes a small amount of training
on the critical path. Most of the overlap benefit is therefore
retained with four cells in this setting. This result concerns cell
granularity at a fixed budget; efficiency under substantially smaller
total GPU budgets remains unevaluated.

\paragraph{Shuffling with streaming training.}
With multiple training steps per burst ($\mu>1$), streaming training
assigns ready rollout groups incrementally to successive updates.
The current design therefore cannot globally shuffle data across all
minibatches before the burst starts. Supporting such a shuffle would
require waiting for the complete burst's data before its first update,
delaying training and sacrificing some overlap with rollout. This
constrains training procedures that depend on globally randomized
minibatch assignment. Single-step bursts ($\mu=1$) avoid the
cross-minibatch split while retaining elastic allocation and streaming.
At matched scheduling staleness, E1 takes 10.9\% longer than E0,
but still achieves a $1.24\times$ speedup over streaming Async (A3)
and a $1.08\times$ speedup over Colocate with streaming (C2)
(Table~\ref{tab:ablation-scheduling}). Thus, multi-step bursts provide
an additional efficiency gain, while single-step bursts remain a
practical option. These ablations quantify execution time; the effect
of sample ordering on learning quality has not been isolated.

%% file: content/conclusion.tex
\section{Conclusion}
\label{sec:conclusion}

We presented \sys{}, which combines elastic scheduling and trajectory
processing for online RL through unmodified harnesses at \xlong{}-horizons. It reallocates GPUs while preserving live executions and
constructs bounded training samples with original contexts and balanced
execution-level contributions. We demonstrate these capabilities through
\xlong{} RL training of our flagship model, Qwen~3.8 2.4T.
Across three workloads, \sys{} achieves up to $1.78\times$ speedup over
Async and $1.85\times$ over Colocate under equal GPU budgets and
matched scheduling staleness, with comparable training scores.

%% file: content/appendix_staleness_analysis.tex
\theoremstyle{definition}
\newtheorem{dispatchassumption}{Assumption}[section]
\theoremstyle{plain}
\newtheorem{dispatchtheorem}{Theorem}[section]

\section{Mean Staleness under Boundary Dispatch}
\label{app:dispatch-analysis}

All three methods in Section~\ref{sec:experiments} use boundary
dispatch. We derive the mean staleness used to match their
configurations and explain why the identity also holds under elastic
GPU placement.

\subsection{Model and Accounting Assumptions}
\label{app:dispatch-model}

Each training step consumes $B$ distinct rollout groups. Policy versions
are indexed by training steps and increase by one after each optimizer
update. A burst contains $\mu$ steps and publishes the weights from its
final step with that step's version number.
Following Section~\ref{sec:background-scheduling},
define a consumed group's staleness as
\begin{equation}
    d=v_t-v_d,
    \label{eq:app-staleness-definition}
\end{equation}
where $v_d$ is the latest published version at admission and $v_t$
is the trainer version immediately before the consuming update.
Admission may precede execution launch
(Section~\ref{sec:burst-dispatch}). This metric includes queueing
and readiness waiting, but does not track the behavior policy of
each model call. Means weight each consumed group equally.

Under boundary dispatch, every admitted group retains its slot until
consumption. The queue is filled at startup and replenished with
$\mu B$ new groups after each burst's publication. Its target depth is
\begin{equation}
    Q_{\mathrm b}=(\varphi+\mu)B,
    \label{eq:app-boundary-depth}
\end{equation}
where $\varphi\geq0$ is extra dispatch in units of $B$ groups.
$Q_{\mathrm b}$ counts all unconsumed groups after refill, including
queued, running, and ready groups. No new groups are admitted within
a burst.

\begin{dispatchassumption}[Stationary group accounting]
\label{ass:app-accounting}
$B$, $\mu$, and $Q_{\mathrm b}$ are fixed. Every admitted group is
eventually consumed exactly once. At corresponding burst boundaries,
the unconsumed population and its relative version ages are stationary,
with finite expected total outstanding age.
\end{dispatchassumption}

\subsection{Boundary Staleness Identity}
\label{app:dispatch-identities}

\begin{dispatchtheorem}[Mean staleness]
\label{thm:app-boundary-staleness}
Under Assumption~\ref{ass:app-accounting},
\begin{equation}
    \mathbb{E}_{\mathrm b}[d]
      =\frac{Q_{\mathrm b}}{B}-\frac{\mu+1}{2}
      =\varphi+\frac{\mu-1}{2}.
    \label{eq:app-mean-scheduling-age}
\end{equation}
\end{dispatchtheorem}

\begin{proof}
After update $j$ in a burst, $Q_{\mathrm b}-jB$ groups remain,
and each gains one unit of age. Replacements enter only after
publication, at age zero. Stationarity equates expected age removed
by consumption with expected age added by updates, so
\[
    \mu B\,\mathbb{E}_{\mathrm b}[d]
      =\sum_{j=1}^{\mu}(Q_{\mathrm b}-jB)
      =\mu Q_{\mathrm b}-\frac{B\mu(\mu+1)}{2}.
\]
Dividing by $\mu B$ and substituting
$Q_{\mathrm b}=(\varphi+\mu)B$ proves the result.
\end{proof}

The proof counts admission and consumption events, so neither FIFO
consumption nor a particular rollout latency distribution is needed.
Moving GPUs between rollout and training changes event timing, but
preserves this identity as long as the boundary-dispatch accounting
is unchanged. It therefore applies to Async, Colocate, and \sys{}.

\subsection{Configurations at a Common Mean}

For a target mean $s$, the identity gives
\begin{align*}
    \varphi &= s-\frac{\mu-1}{2}\geq0, \\
    Q_{\mathrm b} &= B\left(s+\frac{\mu+1}{2}\right).
\end{align*}
The burst size must satisfy $\mu\leq2s+1$, and group counts must be
integral. Larger feasible bursts permit a greater dispatch depth
at the same mean; their performance still depends on placement and
workload latency.

With one step per burst, Async has $s=\varphi$. Without extra
dispatch, Colocate has $s=(\mu-1)/2$. Thus, at $s=1.5$, the
$(\varphi,\mu)$ settings $(1.5,1)$ for Async, $(0,4)$ for Colocate,
and $(0.5,3)$ for \sys{} share the same mean staleness.
The same calculation preserves the mean when the scheduling ablation
changes \sys{} from $(0.5,3)$ to $(1.5,1)$.

%% file: content/appendix_rl_algorithm.tex
\section[RL Algorithm Details]{{RL Algorithm Details}}
\label{app:rl-algorithm}

This appendix specifies the GSPO adaptation summarized in Section~\ref{sec:rl-algorithm}, including its execution-level clipping, token-level importance correction, loss normalization, and training hyperparameters.

\paragraph{{Training configuration.}}
{We use Group Sequence Policy Optimization (GSPO)~\citep{gspo}
with token-level importance correction. Each rollout group contains
$N=16$ independent rollout executions. Each training step
consumes a batch of $B=32$ distinct rollout groups for NL2RepoBench or $B=48$
for DeepSWE and TerminalBench. The remaining RL
hyperparameters are shared across all experiments.

All configurations use Adam~\citep{adam} with a constant learning rate of $10^{-6}$,
$\beta_1=0.9$, $\beta_2=0.95$, $\epsilon=10^{-15}$, weight decay 0.01,
and gradient clipping at norm 1.0. KL and entropy regularization are
disabled.}%

\paragraph{{Execution-level targets.}}
{For execution $i$, let $\mathcal{R}_i$ be its admitted trajectories.
For each $r\in\mathcal{R}_i$, the target set $\mathcal{T}_r$ is fixed to
the trainable tokens still unmasked when $r$ is drawn
(Section~\ref{sec:trajectory-sampling}). On-tree masking makes these
sets disjoint. Let $\mathcal{T}_i=\bigcup_{r\in\mathcal{R}_i}\mathcal{T}_r$,
$L_r=|\mathcal{T}_r|$, and $L_i=|\mathcal{T}_i|=\sum_{r\in\mathcal{R}_i}L_r$.
The sequence used for importance-ratio computation and clipping is this
entire execution's set of admitted, deduplicated trainable tokens.
Each token retains its original generation context. Tokens excluded by
provenance, role, or the $J_{\max}$ cutoff contribute neither to the
sequence ratio nor to the loss.}%

\paragraph{{Importance correction and clipping.}}
{Let $y_u$ be target token $u$, $c_u$ its recorded conditioning
context, and $p_u^{\mathrm{roll}}$ its recorded behavior probability,
reused without recomputation. Let $\operatorname{sg}$ denote
stop-gradient. For $L_i>0$, the token importance ratio and detached
execution-level sequence ratio are
\begin{equation}
  \begin{aligned}
    \rho_u(\theta) &= \frac{\pi_\theta(y_u\mid c_u)}{p_u^{\mathrm{roll}}},\\
    s_i(\theta) &= \operatorname{sg}\!\left[\exp\!\left(\frac{1}{L_i}
      \sum_{u\in\mathcal{T}_i}\log\rho_u(\theta)\right)\right].
  \end{aligned}
  \label{eq:execution-ratio}
\end{equation}
The detached sequence ratio in Equation~\ref{eq:execution-ratio} is
used for monitoring and determines a shared clipping mask according
to execution $i$'s group-relative advantage $\hat A_i$:
\begin{equation}
  m_i(\theta)=
  \begin{cases}
    0, & \hat A_i>0\ \text{and}\ s_i(\theta)>1.005,\\
    0, & \hat A_i<0\ \text{and}\ s_i(\theta)<0.995,\\
    1, & \text{otherwise}.
  \end{cases}
  \label{eq:execution-clip}
\end{equation}
The token-specific importance weight is
$\bar\rho_u(\theta)=\operatorname{sg}(\min\{\rho_u(\theta),5\})$.
These weights may differ within an execution; $s_i$ affects the loss
only through the common mask in Equation~\ref{eq:execution-clip}.
The token loss, upper-bounded by 100, is
\begin{equation}
  \ell_u(\theta)=\min\!\left\{100,\,
    -m_i(\theta)\bar\rho_u(\theta)\hat A_i
       \log\pi_\theta(y_u\mid c_u)\right\}.
  \label{eq:token-loss}
\end{equation}}%

\paragraph{{Loss normalization.}}
{Using the token losses in Equation~\ref{eq:token-loss}, weighting
each trajectory by its target share $L_r/L_i$ and averaging over its
targets gives
\begin{equation}
  \mathcal{L}_i(\theta)
    =\sum_{r\in\mathcal{R}_i}\frac{L_r}{L_i}\cdot\frac{1}{L_r}
       \sum_{u\in\mathcal{T}_r}\ell_u(\theta)
    =\frac{1}{L_i}\sum_{u\in\mathcal{T}_i}\ell_u(\theta),
  \label{eq:trajectory-loss}
\end{equation}
Equation~\ref{eq:trajectory-loss} gives every admitted target the same
normalization coefficient $1/L_i$; importance weighting and clipping
do not change this denominator. Executions with $L_i=0$ contribute zero
loss. A batch averages over all $B$ queries and their $N$ original
rollout executions:
\begin{equation}
  \mathcal{L}(\theta)=\frac{1}{B\,N}\sum_{q=1}^{B}\sum_{k=1}^{N}
     \mathcal{L}_{i(q,k)}(\theta),
  \label{eq:batch-loss}
\end{equation}
where $i(q,k)$ is query $q$'s $k$-th execution. Equation~\ref{eq:batch-loss}
preserves each execution's aggregate normalization weight regardless
of how many trajectories it contributes. For a fixed admitted target
set $\mathcal{T}_i$, changing which selected path carries a shared
target leaves the execution's sequence ratio, clipping decision, and
loss unchanged. Sampling can change $\mathcal{T}_i$ itself; the
invariance concerns the assignment and ordering of a fixed set of
targets. Packing preserves these per-execution weights.}%

%% file: content/appendix_case_study.tex
\raggedbottom
\section{Case Study: One Execution, Multiple Training Trajectories}
\label{app:trajectory-case-study}

We examine execution \texttt{f3ad38e5}, in which an agent uses Claude
Code to build a Python analytics SDK from a natural-language
specification. The task requires nine packages, including catalog
services, data frames, Flight connectivity, and pipeline utilities,
implemented from scratch in a shared workspace.
Figure~\ref{fig:case-prefix-tree} combines its prefix structure with
recorded messages to illustrate the trajectory-processing problem
of Section~\ref{sec:structured-trajectory}.

\begin{figure*}[p]
  \centering
  \includegraphics[width=\textwidth,height=0.88\textheight,keepaspectratio]{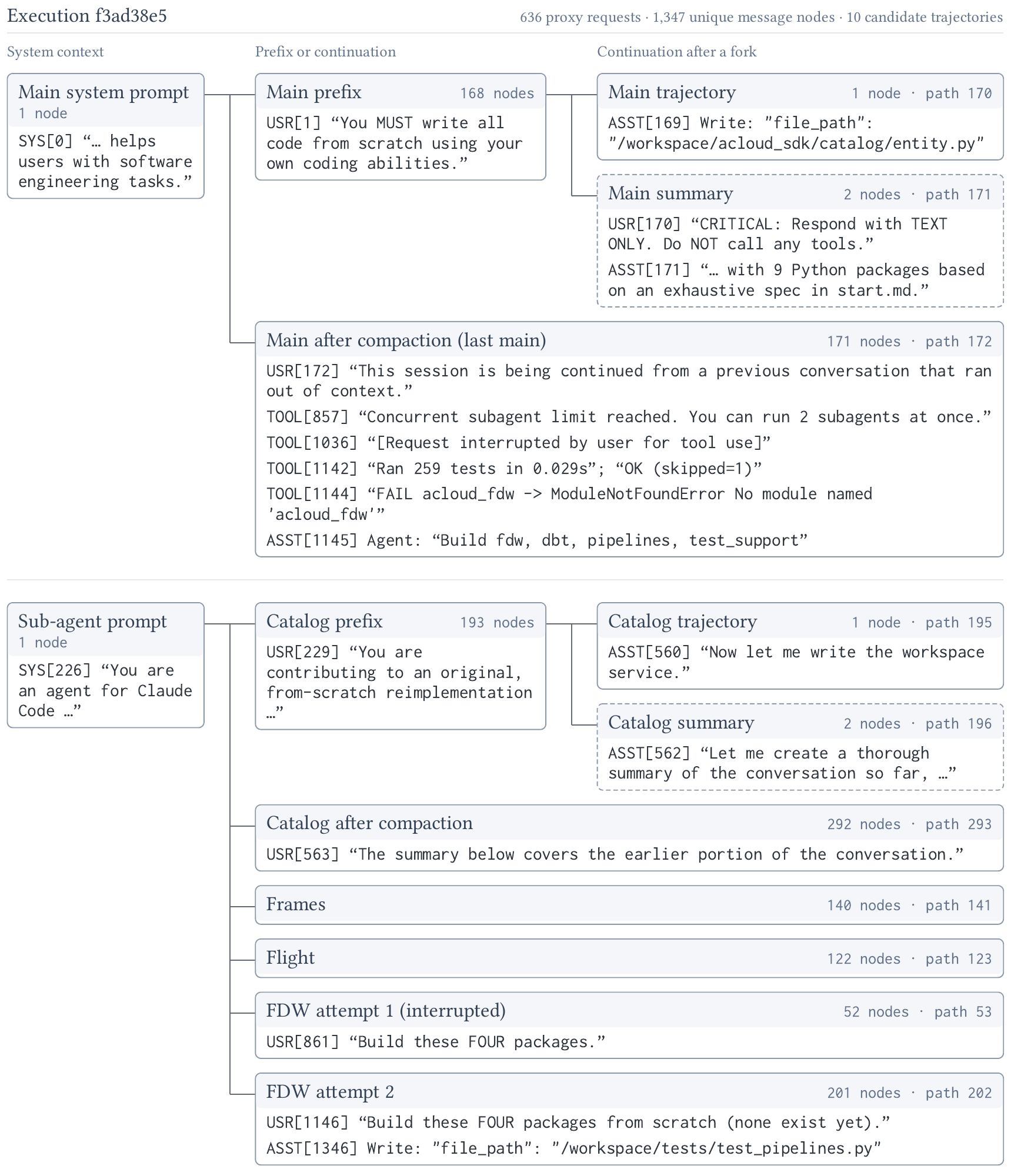}
  \caption{Prefix graph of execution \texttt{f3ad38e5} with real message
    excerpts. Rectangles represent collapsed context segments;
    right-angle edges encode prefix sharing. Each parent is top-aligned
    with its first child, and other children are stacked below.
    Node counts refer to each segment; \emph{path} counts the complete
    root-to-leaf context. Dashed rectangles identify summary leaves.
    Excerpts carry their original roles and node IDs; ellipses mark
    omissions. The interruption at \texttt{TOOL[1036]} remains in the
    main context where it was received.}
  \Description{A full-page node diagram with two system-prompt roots
    and ten trajectory leaves. Hierarchical levels run from left to
    right, with each parent and its first child aligned at their top
    edges. Horizontal and vertical connectors join rectangular nodes.
    Real log excerpts inside the nodes show the coding constraint,
    compaction, a two-agent concurrency limit, an interrupted tool call,
    missing packages after local testing, and renewed FDW delegation.
    Main and catalog prefixes each fork into task and summary paths;
    their compacted continuations start separate paths.}
  \label{fig:case-prefix-tree}
\end{figure*}

\subsection{Execution and Context Structure}

The trace reports 636 proxy requests and 1,347 unique message nodes
across ten candidate trajectories. The main agent produces two
task-driving trajectories and one summary. Five sub-agent invocations
produce six task-driving trajectories and one summary: the catalog
sub-agent compacts its history, while the FDW work item has two
attempts. These visualization nodes count messages, distinct from
the data plane's request increments and trainable tokens.

Main and sub-agent contexts have separate roots because their system
prompts differ. Delegation can be identified separately by matching
an \texttt{Agent} call's prompt to a sub-agent's initial user message
after removing harness reminders; the export also identifies a
sub-agent billing flag. Causal parentage therefore differs from
context-prefix sharing. Concatenating the two agents' conversations
would give outputs conditioning messages absent during generation.

\subsection{Compaction and Shared Prefixes}

Both the main agent and the catalog sub-agent compact their contexts.
After their system prompts, they have shared prefixes of 168 and 193
messages, respectively. Each prefix forks into a task-driving
assistant response and a summarization request--response pair.
Continuation after compaction starts a separate path under the
system prompt, with a new user message containing the summary.

The ten complete paths contain 53--293 message nodes each.
Expanding them independently produces 1,716 node occurrences,
compared with 1,347 shared nodes. The 369 extra occurrences repeat
shared prefixes and system prompts; their token volume depends on
message lengths.

If a task-driving leaf and its summary leaf are both selected,
their shared policy outputs must contribute loss once.
Section~\ref{sec:trajectory-sampling}'s on-tree masking retains the
prefix as context while suppressing repeated targets. Provenance
also matters: the original summary can be a target if the rollout
policy generated it, whereas its reinsertion into the continuation's
user prompt is conditioning input. TITO preserves the original token
contexts and distinguishes these two occurrences.

\subsection{Delegation, Interruption, and Outcome Validity}

One response attempts four \texttt{Agent} calls, but the harness
admits only two concurrent sub-agents. Two tool results report the
concurrency limit. Subsequent calls start the remaining work, but
the first FDW attempt is interrupted after a 53-node path. Later,
a local test run reports \texttt{OK} for 259 tests, with one skipped,
yet an import check cannot find the FDW, dbt, pipeline, and
test-support packages. The main agent launches another attempt,
which contributes a separate 202-node path. Rejection, partial
execution, and retry thus affect the record differently within
one rollout.

The leaf tagged \texttt{is\_final\_main} ends with renewed delegation:
it identifies the last recorded main trajectory, not successful
completion. An interrupted sub-agent likewise does not determine
the execution's reward. Section~\ref{sec:trajectory-partial-score}
requires the designated evaluator to assess the preserved workspace;
a group becomes trainable only after its executions close with valid
outcomes. The visualization supplies no designated task reward.
The recovery sequence also motivates preserving harness and workspace
state during the request rerouting of
Section~\ref{sec:elastic-execution}.

\subsection{Applying the Training Admission Rule}

The priority classes are two \emph{main}, one \emph{main-summary},
six \emph{sub-agent}, and one \emph{sub-agent-summary}.
For illustration, suppose $J_{\max}=4$ and each path has eligible
policy-generated targets. The rule in
Section~\ref{sec:trajectory-sampling} admits both main trajectories,
the main-summary trajectory, and one sub-agent trajectory.
Selection within each class is proportional to remaining trainable
tokens, with shared targets masked after every draw. The figure's
message counts cannot determine these token probabilities.

A flat mean over all ten leaf losses would allocate eight equal
coefficients to delegated or summary paths and repeat shared
targets. In \sys{}, selected paths inherit the original execution's
group-relative advantage. Token losses are averaged over disjoint
targets within each execution, then over the original rollout executions.
Additional compaction or delegation changes the admitted targets
without creating independent rewards or increasing the execution's
aggregate normalization weight.